\documentclass[conference]{IEEEtran}
\usepackage{blindtext, graphicx}
\usepackage{graphicx}
\usepackage{algorithm}
\usepackage{algorithmic}
\usepackage{amsthm}
\usepackage{amssymb}
\usepackage{multicol}

\usepackage{times}

\newtheorem{thm}{Theorem}

\newtheorem{fc}{Fact}

\newtheorem{cor}{Corollary}
\newtheorem{hyp}{Hypothesis}
\newtheorem{pro}{Property}

\ifCLASSINFOpdf
\else
\fi
\begin{document}
%
\title{Spectral Clustering using Eigenspectrum Shape Based Nystr\"{o}m Sampling }

\author{\IEEEauthorblockN{Djallel Bouneffouf}
\IEEEauthorblockA{IBM Thomas J. Watson Research Center,\\ Yorktown Heights, NY USA,\\
Email: Djallel.bouneffouf@ibm.com\\
}
}


%


\maketitle

\begin{abstract}
Spectral clustering has shown a superior performance in analyzing the cluster structure. However, its computational complexity limits its application in analyzing large-scale data. To address this problem, many low-rank matrix approximating algorithms are proposed, including the Nystr\"{o}m method -- an approach with proven approximate error bounds. There are several algorithms that provide recipes to construct Nystr\"{o}m approximations with variable accuracies and computing times. This paper proposes a scalable Nystr\"{o}m-based clustering algorithm with a new sampling procedure, Centroid Minimum Sum of Squared Similarities (CMS3), and a heuristic on when to use it. Our heuristic depends on the eigenspectrum shape of the dataset, and yields competitive low-rank approximations in test datasets compared to the other state-of-the-art methods.\end{abstract}


%
\IEEEpeerreviewmaketitle

\section{Introduction}
Clustering is one of the fundamental problems in machine learning \cite{fowlkes2004spectral}. The recent development of data-storage and data-acquisition devices has increased the scale of data sets, which poses a serious computational challenge for the existing offline and online learning learning algorithm \cite{BouneffoufBG13,ChoromanskaCKLR19,RiemerKBF19,LinC0RR20,lin2020online,lin2020unified,NoothigattuBMCM19,surveyDB,LR85,Bouneffouf0SW19,LinBCR18,DB2019,BalakrishnanBMR19ibm,BouneffoufLUFA14,RLbd2018,balakrishnan2020constrained,BouneffoufRCF17,BalakrishnanBMR18,BouneffoufBG12,Bouneffouf16,aaai0G20,AllesiardoFB14,dj2020,Sohini2019,bouneffouf2020online,bouneffouf2020contextual} . Spectral clustering techniques are widely used, due to their empirical performance advantages compared to other clustering methods \cite{kong2011fast}. However, a significant obstacle to scaling up spectral clustering to large datasets is that it requires building an affinity matrix between pairs of data points which becomes computationally prohibitive for large data-sets \cite{ChenC11}.

To address this computational challenge, a common approach is to use the Nystr\"{o}m method as low-rank matrix approximation \cite{zhang2011clusterability}, \cite{fowlkes2004spectral}. \cite{williams2001using}. The method works by sampling a small set of landmark points from a large instances, to formulate an approximation for the eigen-decomposition of the full dataset using the sampled data. 
However, the performance of the approach is highly dependent on proper sub-sampling of the input data to include some {\em landmark points}, points that capture the inherent complexity and variability of the full dataset. Uniform sampling without replacement is the most used approach for this purpose \cite{fowlkes2004spectral}, \cite{cohen2014uniform}. 

Using local or global properties of the data distribution a leading version of non-uniform sampling has recently been suggested. The authors in \cite{bouneffouf2016ensemble}, propose the ensemble minimum sum of the squared similarity sampling algorithm or ensemble-MS3. This algorithm is based on two works, the first one is the minimum sum of the squared similarity sampling or MS3 proposed in \cite{bouneffouf2015sampling}, that considers both the variance and the similarity of the dataset to select the landmark points. The second one is the ensemble Nystr\"{o}m methods proposed in \cite{kumar2009ensemble}, which is a meta algorithm that combines the standard Nystr\"{o}m methods with the mixture weights. The ensemble-MS3 gives better results than the standard algorithms by increasing the accuracy compared with the standard Nystr\"{o}m method. However, the lack of speed is still a problem for the ensemble methods since the algorithm need to sample multiple times in order to aggregate the results.

In this paper, we propose two algorithms that perform better than the ensemble MS3 and any existing ensemble Nystr\"{o}m algorithm. The first one, the ''Centroid Minimum Sum of Squared Similarities algorithm" or CMS3 is an incremental sampling algorithm for Nystr\"{o}m based-spectral clustering. CMS3 improves the MS3 by adding centroid sampling upon the MS3, increasing the accuracy. In the first step, the algorithm starts sampling with a fixed number of initial landmark points and selects new landmark points one by one, such that the sum of the squared similarities between the previously selected points and the new point is minimized, and as a second step the algorithm selects only the centroid points from this sub-sample. 
The second one, the CMS3-tuned is deducted from the theoretical analyse of MS3 and leads to adapt the sampling according to the spectrum shape of the dataset.

The rest of the paper is organized as follows: related works are discussed in Section 2. In Section 3, we briefly introduce spectral clustering and the Nystr\"{o}m extension. An error analysis of the approximated matrix is proposed in Section 4, and an incremental sampling method is also proposed. We evaluate the proposed method in Section 5. Finally, we conclude our work in Section 6.

\section{Related Work}
 \label{section:RW} 
 
To apply spectral clustering to large datasets, new efforts have been concentrating on solving issues around algorithm scalability \cite{zeng2014minimum}, such as using dimension reduction by Nystr\"{o}m approximation \cite{williams2001using}, a method originally designed for numerical solution of integral equations \cite{sloan1981quadrature}. However the performance of the approach is highly dependent on proper landmark points that capture the inherent variability of the full dataset \cite{Loo2014}. 

To address this problem, different sampling methods have been proposed, assuming that clusters have an equiprobable distribution, authors in \cite{fowlkes2004spectral} and \cite{cohen2014uniform} propose a random sampling (RS). Although, this implicit assumption is not true in all datasets, it is shown in \cite{kumar2009sampling} that it performs better than two proposed alternatives that use diagonal sampling \cite{drineas2005Nystrom} and column-norm sampling \cite{drineas2006fast} algorithms.

In \cite{belabbas2009spectral} the authors developed a weighted sampling (WS) approach using the determinant of the kernel matrix to select landmark points, where the probability of choosing a new landmark point was in proportion to the determinant of the similarity matrix between landmark points. They analyzed the Nystr\"{o}m reconstruction error using the Schur complement \cite{gowda2010schur}, concluding that the larger the determinant, the smaller the error. Although the work provides a solid theoretical basis for measuring the error levels in Nystr\"{o}m approximation, a main drawback of the algorithm provided is in its time complexity.

Assuming that the potential clusters are convex, \cite{zhang2008improved} introduced $k$-means based sampling (KS) algorithm, as a means to select points near $k$-means centroids as landmark points. Similarly, \cite{shinnou2008spectral} also pre-processed the data using $k$-means clustering, to select a {\em committee} of data points near centroids. Although the latter method does not explicitly state the convexity assumption, both methods perform poorly for non-convex clusters.

In \cite{zhang2011clusterability}, the authors proposed an incremental sampling (IS) algorithm that first randomly samples two points from a dataset, to compute a similarity matrix between the sampled points and the remaining points. The algorithm picks the point with the smallest variance, and then iteratively repeats the process until a desired number of landmarks is reached. While promising, \cite{zeng2014minimum} showed that IS performs poorly on high-dimensional data, as the variance of the Euclidean distance tends to zero. In such cases IS may pick inappropriate landmark points for dimension reduction, hence for successful clustering.

The minimum similarity sampling (SS) is proposed in \cite{zeng2014minimum} for high-dimensional space clustering purpose. The authors studied how the similarity between the sample set and non-sample set influences the approximation error, and observed that their result depends on the dimensionality of the dataset: SS outperforms IS on high-dimensional data, but not on low dimensional data.

Recently, a new sampling algorithm is proposed in \cite{bouneffouf2015sampling}, named MS3 for Minimum Sum of Squared Similarities, this algorithm approximately maximizes the determinant of the reduced similarity matrix that represents the mutual similarities between sampled data points, and demonstrates the performance of MS3 compared with the standard Nystr\"{o}m method. An ensemble version of MS3 method was proposed in \cite{bouneffouf2016ensemble}. It treats each approximation generated by the MS3 method for a sample of columns as an expert and combines such experts to derive an improved hypothesis, typically more accurate than any of the original experts, but the drawback of this method is in its computationally time.

In this paper, we propose the CMS3 that performs better than the ensemble MS3 algorithm. The proposed algorithm samples at first using MS3 and after that selects only the centroid points of the MS3 sampling. We have also proposed an improved version based on theoretical analysis of the upper error bound of this algorithm. This tuned version yields more accurate low-rank approximations than the ensemble Nystr\"{o}m methods. 

\section{Key Notion}
This section focuses on introducing the key notions used in this paper.
\subsection{Spectral Clustering}\label{ss:SC}
Spectral clustering algorithms employ the first $k$ eigenvectors of a Laplacian matrix to guide clustering. Loosely following the notation in \cite{von2007tutorial}, this can be outlined as follows.
The algorithm takes as an input a number $k$ of clusters, an affinity matrix $S \in R^{n \times n}$ constructed using the cosine similarity between each pairs of data points, and as an output clusters $c_1, . . ., c_k$. It starts by computing the Laplacian matrix $P = D-S$ ; where $D$ is an $n\times n$ diagonal matrix defined by $D_{ii} = \sum_{j=1}^n S_{ij}$, and after that it computes $k$ eigenvectors $u_1 ,. . .,u_k$ corresponding to the first $k$ eigenvalues of the generalized eigenproblem $Pu = \lambda Du$; and let $Z\in R^{n \times k}$ be the matrix containing the vectors $u_1, . . ., u_k$. Finally, it clusters $y_1, . . ., y_n$ by k-means algorithm into clusters $c_1, . . ., c_k$; with $y_i$ corresponding to the $i$-th row of $Z$.

By analyzing the spectrum of the Laplacian matrix constructed over all data entries, the original data can be compressed into a smaller number of representative points using the Nystr\"{o}m approximation described below.

\subsection{Nystr\"{o}m Sampling}
If we consider $m$ landmark data points $L = { l_1,l_2,...,l_m}$ from a given dataset $X ={x_1,x_2,...,x_n}$ with $x_i \in R^n$ and $m \ll n$,
then for any given point $x$ in $X$, Nystr\"{o}m method formulates
{\small\begin{equation}
 \label{eq:pr1}
{1 \over m}\sum\limits_{i = 1}^m sim (x,{l_i})\hat \phi ({l_i}) = \hat\lambda \hat \phi (x)
\end{equation}}
where $\hat \phi(x)$ is an approximation to the exact eigenfunction, $\hat\lambda$ is the corresponding approximate eigenvalue and $sim(x,y)$ denotes the similarity between $x$ and $y$.

We can write the Eq.\ref{eq:pr1} in matrix form, $\widetilde{S}\hat \Phi = m\hat \Phi \hat \Lambda$ where $\hat \Phi=[\hat \phi_1 \hat \phi_2...\hat \phi_m]$ are the eigenvectors of $\widetilde{S}$ and
$\hat\Lambda= diag \{\hat \lambda_1,\hat \lambda_2,\ldots,\hat \lambda_m \}$ is a diagonal matrix of the corresponding approximate eigenvalues. Then for an unsampled point $x$, the $j$-th eigenfunction at $x$ can be approximated by $
{\hat \phi _j}(x) \simeq{1 \over {m{{\hat \lambda }_j}}}\sum\limits_{i = 1}^m sim (x,{l_i}){\hat \phi _j}({l_i})$. With this equation, the eigenvector for any given point $x$ can be approximated through the eigenvectors of the landmark points $L$ \cite{belabbas2009spectral}. The same idea can be applied to approximate $k$ eigenvectors of $S$ by decomposing and then extending a $k \times k$ principal sub-matrix of $S$. First, let $S$ be partitioned as $S=\left[{\matrix{ A & B^\top & \cr B & C \cr } } \right]$ with $A\in R^{k \times k}$. Now, define spectral decompositions $S=U \Lambda U^T$ and $A=U_A \Lambda_A U^T_A$; the Nystr\"{o}m extension then provides an approximation for $k$ eigenvectors in $\widetilde{U}=\left[{\matrix{ U_A & \cr BU_A \Lambda_A^{-1} \cr } } \right]$
where the approximations of $\widetilde{U} \approx U$ and $\widetilde{\Lambda} \approx \Lambda$ may then be composed, yielding an Nystr\"{o}m approximation $\widetilde{S} \approx S$, with $\widetilde{S}=\widetilde{U}\Lambda_A\widetilde{U}^{\top}$. To measure the distance of these approximations, conventionally Frobenius norm is used.

\subsection{Minimum Sum of Squared Similarities}\label{ss:M3S}
The MS3 algorithm \cite{bouneffouf2015sampling} initially randomly chooses two points from the dataset $X$. It then computes the sum of similarities between the sampled points and a subset, $T$, selected randomly from the remaining data points. The point with the smallest sum of squared similarities is then picked as the next landmark data point. The procedure is repeated until a total of $m$ landmark points are picked.

\begin{algorithm}
 \caption{The Minimum Sum of Squared Similarities Algorithm } 
\label{alg:MSSS} 
\begin{algorithmic}[1]
 \STATE {\bfseries }\textbf{Input:} 
 $X=\{x_1,x_2,...,x_n\}$: dataset \\
 $m$: number of landmark data points\\
 $\gamma$: size of the random sub-sampled set from the remaining data, in percentage \\ 
 \STATE {\bfseries }\textbf{Output:} $\widetilde{S} \in R^{m \times m}$: similarity matrix between landmark points
 \STATE {\bfseries } Initialize $\widetilde{S}=I_0$
 \STATE {\bfseries } \textbf{For (i$=$0 to i$<$2) do}
 \STATE {\bfseries } \hspace{2em} $\widetilde{x}_i=Random(X)$
 \STATE {\bfseries } \hspace{2em} $\widetilde{S}:=\widetilde{S}_{\cup x_i} $
 \STATE {\bfseries } \hspace{2em} $\widetilde{X}:=\widetilde{X} \cup \{\widetilde{x}_i\} $
 \STATE {\bfseries } \textbf{End For} 
 \STATE {\bfseries } \textbf{While $i<m$ do}
 \STATE {\bfseries } \hspace{2em} $T=Random(X\backslash \{\widetilde{X}\},\gamma)$
 \STATE {\bfseries } \hspace{2em} Find $\widetilde{x}_{i}=argmin_{x \in T} \sum_{j<i-1} sim^2(x,\widetilde{x}_j)$
 \STATE {\bfseries } \hspace{2em} $\widetilde{S}:=\widetilde{S}_{\cup \widetilde{x}_i} $
 \STATE {\bfseries } \hspace{2em} $\widetilde{X}:=\widetilde{X} \cup \{\widetilde{x}_i\} $
 \STATE {\bfseries }\textbf{End While}  
 \end{algorithmic}
\end{algorithm}

\section{Centroid Minimum Sum of Squared Similarities (CMS3)}
The idea of the proposed algorithm CMS3 (described in Algorithm \ref{alg:CM3S}) is to sample $r$ points using MS3 where $m \leq r \leq X$ with the assumption that this sampling will give an $r$ convex points, and after that the CMS3 uses k-means \cite{macqueen1967some} to cluster these $r$ points and select the centroids of these clusters as a global optimal landmark points. 
We could say that, the proposed algorithm is implemented under the following Hypothesis:
\begin{hyp}
\label{hyp:cm3s}
Comparing two similarity matrix $\widetilde{S}_{m}$ and $\widetilde{S'}_{m}$ corresponding to $CMS3$ and $MS3$ approximations, we have the following inequality between their error upper bounds:
\begin{eqnarray*}
 sup(||S-\widetilde{S}_{m}||) \leq sup(||S-\widetilde{S'}_{m}||)
\end{eqnarray*}
\end{hyp}
\begin{algorithm}[h!]
 \caption{CMS3 Algorithm}
\label{alg:CM3S} 
\begin{algorithmic}[1]
 \STATE {\bfseries }\textbf{Input:} 
 $X=\{x_1,x_2,...,x_n\}$: dataset \\
 $m$: number of landmark data points\\
 $r$: number of landmark data points selected with MS3\\
 $\gamma$: size of the random subsampled set from the remaining data, in percentage \\ 
 \STATE {\bfseries }\textbf{Output:} $\widetilde{S} \in R^{m \times m}$: similarity matrix between landmark points
 \STATE {\bfseries } Initialize $\widetilde{S}=I_0$
 \STATE {\bfseries } $X_r:=MS3(X, r, \gamma)$
 \STATE {\bfseries } $\widetilde{r}:= kmeans(X_r,m)$
 \STATE {\bfseries } \textbf{For (i$=$0 to i$\leq $ m) do}
 \STATE {\bfseries } $\widetilde{x}_i:= \frac{1}{|\widetilde{r}_i|}\sum_{x_j \in \widetilde{r}_i}x_j$ //get centroid of the cluster $\widetilde{r}_i \in \widetilde{r}$
 \STATE {\bfseries } \hspace{1em} $\widetilde{S}:=\widetilde{S}_{\cup \widetilde{x}_i} $
 \STATE {\bfseries } \textbf{End For} 
 \end{algorithmic}
\end{algorithm}

\subsection{Theoretical Study}
We propose here to study under which condition the proposed \textit{Hypothesis \ref{hyp:cm3s}} is valid.
In order to do that, we propose at first to compute the the upper bound of the proposed sampling algorithm "CMS3" in \textit{Theorem \ref{thm:boundkm3s}} and then compare it to the "MS3" upper bound in \textit{Corollary \ref{cor:msss0}}.

{\small
\begin{thm}\label{thm:boundkm3s}
For a dataset $X=\{x_1,x_2,...,x_n\}$, define the following positive definite similarity matrices:
\begin{itemize}
\item $S$: the $n \times n$ similarity matrix for the overall dataset with a maximum diagonal entry $S_{\max}$;
\item $\widetilde{S}_l$: a similarity matrix for $X_l$ with $l$ landmark point selected randomly from $X$; 
\item $\widetilde{S}_r$: a similarity matrix for $X_r$ with $r$ landmark point selected from $X_l$ using MS3, with $r \leq l \leq n$; 
\item $\widetilde{S}_m$: a similarity matrix for $X_m$ with $m$ landmark point selected from $X_r$ using K-means sampling, with $m \leq r \leq l \leq n$; 
and
\item $S_k$: the best possible rank-$k$ approximation of $S$. 

\end{itemize}
Then with some probability $1-p$ or more, we can write
\begin{eqnarray}\label{eq:pr7}
||S-\widetilde{S}_m|| &\leq & 4T \sqrt{mC^{kern}_XTe}+mC^{kern}_XTe||W^{-1}||\nonumber\\
 &+&(r+1) \sum_{i=r+1}^n \lambda_i + ||S-S_k|| \\
 &+& nS_{\max}\sqrt[4]{\frac{64k}{l}} \left(1+ \sqrt{\frac{w d^*_S}{S_{\max}}} \right)^{\frac{1}{2}}\nonumber
\end{eqnarray}
where
$||.||$ is the Frobenius norm.
\[d^*_S=\max_{ij} \left(S_{ii}+S_{jj}2S_{ij}\right)\]
and
\[w=-\frac{n-1} {2n-1} \frac{2}{\beta(l,n)} \log p\]
with 
\[\beta(l,n)=1-\frac{1}{2\max\{l,n-l\}}\]
\end{thm} 
} 
\begin{proof}
Using the above notation, let us introduce some facts.

{\small\begin{fc} \cite{bouneffouf2015sampling}
\label{fc:bound}
Let $\lambda_1 \ge...\ge \lambda_n$ be the eigenvalues of the similarity matrix $S$, then with some probability $1-p$ or more, we can write
{\small\begin{eqnarray}\label{eq:prboune}
||S-\widetilde{S}_r|| &\leq& (r+1) \sum_{i=r+1}^n \lambda_i + ||S-S_r|| \\
 &+& nS_{\max}\sqrt[4]{\frac{64k}{l}} \left(1+ \sqrt{\frac{w d^*_S}{S_{\max}}} \right)^{\frac{1}{2}}\nonumber
\end{eqnarray}}
\end{fc} }

{\small\begin{pro} \cite{zhang2008improved}
\begin{eqnarray}\label{eq:kernel}
 (kern(a,b)-kern(c,d))^2 \leq \nonumber
C^{kern}_X(||a-c||^2+||d-b||^2),\\
\forall a, b, c, d \in R \nonumber
\end{eqnarray}
 where $C^{kern}_{X}$ is a constant depending on, the kernel $kern(.,.)$ and the sample set $X$.
 \end{pro}}
{\small\begin{fc} \cite{zhang2008improved}
Let the whole sample set $X$ be partitioned into $g$ disjoint clusters $S_{kern}$, $c(i)$ being the function that maps each sample $x_i \in X$ to the closest landmark point $z_{c(i)}\in Z$. Then for some kernel $kern$ satisfying property (1), the partial approximation error $||S-\widetilde{S}_m||$
is bounded by 
\begin{equation}\label{eq:fact3}
||S-\widetilde{S}_m|| \leq 4T \sqrt{mC^{kern}_XTe}+mC^{kern}_XTe||W^{-1}||
\end{equation}
where $T=max_{kern}|S_{kern}|$, and $e$ is the quantization error induced by coding each sample in $x_{i} \in X$ by the closest landmark point in $Z$, i.e., 
$e=\sum_{x_i\in X}||x_i-z_{c(i)}||^2$, 
and $||W^{-1}|| \in R^{m \times m}$ where $w_{ij}=k(z_i,z_j)$.
\end{fc}}
 
By adding both sides of Eq.\ref{eq:prboune} and Eq.\ref{eq:fact3}, noting that 
$\sum_{i=m+1}^n(.)\ge\sum_{i=r+1}^n(.)$ for positive argument and using the triangle inequality 
\begin{equation}\label{eq:pr8}
||S-\widetilde{S}_m|| \leq ||S-\widetilde{S}_r|| + ||\widetilde{S}_r-\widetilde{S}_m|| 
\end{equation}

we prove Theorem \ref{thm:boundkm3s}.
\end{proof}

{\small\begin{cor}
\label{cor:msss0} 
The proposed \textit{Hypothesis \ref{hyp:cm3s}} is valid, if and only if
\begin{eqnarray}
\label{eq:cor1}
m \leq
\frac{ \lambda_r - r\sum_{i=r+1}^{n} \lambda_i-4T \sqrt{(r-1) C^{kern}_XTe} }{C^{kern}_XTe||W^{-1}||-\sum_{i=r}^n \lambda_i}
\end{eqnarray}

\end{cor}}

\begin{proof}
Assuming the comparison of the upper bounds appears with the inequality,
{\small\begin{eqnarray*}
\label{eq:A}
 4T \sqrt{mC^{kern}_XTe} 
 +mC^{kern}_XTe||W^{-1}||\\
 +(r+1) \sum_{i=r+1}^n \lambda_i + ||S-S_{k}||\\
 + nS_{\max}\sqrt[4]{\frac{64k}{l}} \left(1+ \sqrt{\frac{w d^*_S}{S_{\max}}} \right)^{\frac{1}{2}}
 \end{eqnarray*}}
 {\small\begin{eqnarray}
 \leq (m+1) \sum_{i=m+1}^n \lambda_i + ||S-S_{k}|| \\
 + nS_{\max}\sqrt[4]{\frac{64k}{l}} \left(1+ \sqrt{\frac{w d^*_S}{S_{\max}}} \right)^{\frac{1}{2}}\nonumber
\end{eqnarray}}

after simplification we get

{\small\begin{eqnarray*}
4T \sqrt{mC^{kern}_XTe}+mC^{kern}_XTe||W^{-1}|| &\leq \\
(m+1) \sum_{i=m+1}^n \lambda_i-(r+1) \sum_{i=r+1}^n \lambda_i
\end{eqnarray*}}

Knowing that $m \leq r$ we can write

{\small\begin{eqnarray}
\label{eq:corcond1}
4T \sqrt{mC^{kern}_XTe}+mC^{kern}_XTe||W^{-1}|| &\leq \\
(m-r) \sum_{i=r+1}^n \lambda_i+(m+1) \sum_{i=m+1}^{r} \lambda_i \nonumber
\end{eqnarray}}

which gives

{\small\begin{eqnarray}
\label{eq:B}
m \leq
\frac{ \sum_{i=m+1}^r \lambda_i - r\sum_{i=r+1}^{n} \lambda_i -4T \sqrt{mC^{kern}_XTe} }{C^{kern}_XTe||W^{-1}||-\sum_{i=m+1}^n \lambda_i}
\end{eqnarray}}
then replacing $m$ by $r-1$ gives, 
{\small\begin{eqnarray}
\label{eq:B1}
m \leq
\frac{ \lambda_r - r\sum_{i=r+1}^{n} \lambda_i-4T \sqrt{(r-1) C^{kern}_XTe} }{C^{kern}_XTe||W^{-1}||-\sum_{i=r}^n \lambda_i}
\end{eqnarray}}
We note that going from inequality (\ref{eq:A}) back to (\ref{eq:B1}) is straightforward, and can be achieved by tracing the above steps in reverse.
\end{proof}

\subsection{CMS3-tuned}
We propose here to use the above theoretical results to propose an improved version of the CMS3. 

Corollary \ref{cor:msss0} prescribes a method to select between $MS3$ and $CMS3$ methods. However, due to its complexity, the idea here is to relax the ''if and only if'' of the Corollary \ref{cor:msss0} as follows: 
{\small\begin{cor}
\label{con:msss}
Comparing the upper bound of $MS3$ and $CMS3$, as defined in \textit{Hypothesis \ref{hyp:cm3s}}. Assuming that $m \lambda_{m+1} + r \lambda_n <<\lambda_2$, a necessary condition for $sup(||S-\widetilde{S}_{m}||) \leq sup(||S-\widetilde{S'}_{m}||)$ is
 \begin{eqnarray*}
\lambda_2 &\leq n \lambda_n
\end{eqnarray*}
\end{cor}}

\begin{proof}
From Eq. (\ref{eq:corcond1}) a necessary condition for having the Corollary \ref{cor:msss0} could be the following:
\begin{eqnarray*}
0 &\leq (m-r) \sum_{i=r+1}^n \lambda_i+(m+1) \sum_{i=m+1}^{r} \lambda_i
\end{eqnarray*}
then the following still hold,
\begin{eqnarray*}
0 &\leq (m-r) (n-r) \lambda_n+(m+1) (r-m) \lambda_{m+1}
\end{eqnarray*}

which implies 
\begin{eqnarray*}
0 &\leq (m+1) \lambda_{m+1}-(n-r) \lambda_n
\end{eqnarray*}

with $ \lambda_{m+1} \leq \lambda_2$ we get
\begin{eqnarray*}
0 &\leq (m+1) \lambda_{2}-(n-r) \lambda_n
\end{eqnarray*}

and assuming that $m \lambda_{m+1} + r \lambda_n <<\lambda_2$, gives us
\begin{eqnarray*}
\lambda_2 &\leq n \lambda_n
\end{eqnarray*}

\end{proof}

Following the Corollary \ref{con:msss}, the idea in the proposed algorithm (Algorithm \ref{alg:CM3Stuned}), is to use $\lambda_2 \leq |sm| \times \lambda_{|sm|}$ as a switch condition for using CMS3 or MS3, where $|sm|$ is the sub-sampling size. These parameters could be seen as a proxy of the eigenspectrum shape of the data. 
\begin{algorithm} [H]
 \caption{CMS3-tuned Algorithm}
\label{alg:CM3Stuned} 
\begin{algorithmic}[1]
 \STATE {\bfseries }\textbf{Input:} 
 $X=\{x_1,x_2,...,x_n\}$: dataset \\
 $m$: number of landmark data points\\
 $r$: number of landmark data points selected with MS3\\
 $\gamma$: size of the random subsampled set from the remaining data, in percentage \\ 
 \STATE {\bfseries }\textbf{Output:} $\widetilde{S} \in R^{m \times m}$: similarity matrix between landmark points
 \STATE {\bfseries } \hspace{2em} $sm=Random(X,\gamma)$
 \STATE {\bfseries } \hspace{2em} Compute $|sm|$ eigenvalues $\lambda_1 ,. . .,\lambda_{|sm|}$ of the generalized eigenproblem $Pu = \lambda Du$; and let $Z\in R^{n \times |sm|}$ be the matrix containing the vectors $u_1, . . ., u_{|sm|}$.
 \STATE {\bfseries } \hspace{2em} if $|sm| \times \lambda_{|sm|} \ge \lambda_2$
 \STATE {\bfseries } \hspace{2em} then $\widetilde{S}:=CMS3(X, m, r, \gamma)$
 \STATE {\bfseries} \hspace{2em} else $\widetilde{S}:=MS3(X, m, \gamma)$
 \end{algorithmic}
\end{algorithm}

\begin{thm}\label{thm:boundkm3stuned}
Let $A$ be an $n \times n$ symmetric matrix with eigenvalues $\lambda_1 \geq...\geq \lambda_{n}$. Let $B$ be an $(n-k) \times (n-k)$ symmetric minor of with eigenvalues $ \mu_1 \geq ...\geq \mu_{n}$. Then
then with probability 1-$\delta$ and with $\epsilon = \sqrt{\frac{1}{2|sm|}ln(\frac{1}{\delta})}$ we can write 
\begin{equation}\label{eq:lesslemma}
|(\lambda_2-n \lambda_n)-(\mu_2-m\mu_m)| \leq \epsilon
\end{equation}

\end{thm}
The Theorem \ref{thm:boundkm3stuned} is proving that behavior of our eigenvalue in our sample set is similar to the original set, with some probability.
\begin{proof}
We start by introducing the following fact.
\begin{fc} \label{fc:printerlacing} (Interlacing eigenvalues) Let $A$ be an $n \times n$ symmetric matrix with eigenvalues $\lambda_1 \geq...\geq \lambda_{n}$. Let B be an $(n-k) \times (n-k)$ symmetric minor of with eigenvalues $ \mu_1 \geq ...\geq \mu_{n}$. Then
$ \lambda_i \leq \mu_i \leq \lambda_{i-k}$
\end{fc}

by using Fact \ref{fc:printerlacing} in equation \ref{eq:lesslemma} we get:
\begin{eqnarray}
|\lambda_2-n \lambda_n-\mu_2+m\lambda_m| \leq \epsilon
\end{eqnarray}
Assuming that $\lambda_i-\lambda_{i-1} \geq \lambda_j-\lambda_{j-1} \forall (i<j) $ we have,
\begin{eqnarray}
n \lambda_n < m \lambda_m
\end{eqnarray}
then 
\begin{eqnarray}
|\lambda_2-n \lambda_n-\lambda_{2-k}+m\lambda_m| \leq \epsilon
\end{eqnarray}
\begin{fc} \label{fc:prhofding} (Chernoff-Hoeffding bound) Let $X_i\in [0,1]$ an independent random variables with $\mu=E[X_i]$. 
$Pr(\frac{1}{n} \sum^{n}_{i=1} X_i-\mu \geq \epsilon)\leq e^{-2n\epsilon^2}$
\end{fc} 

First of all, the equation 12 is true if $Pr(|(\lambda_2-n \lambda_n)-(\mu_2-m\mu_m)| \leq \epsilon)$ with $\epsilon$ small. Then to quantify the $\epsilon$ we use the Fact \ref{fc:prhofding} that gives us,

$Pr(|(\lambda_2-n \lambda_n)-(\mu_2-m\mu_m)| \leq \epsilon)\leq e^{-2|sm|\epsilon^2}$. 

Then with $e^{-2|sm|\epsilon^2} = \delta $ we get $\epsilon = \sqrt{2|sm|ln(\frac{2}{\delta})}$.
\end{proof}

\begin{cor}\label{cor:boundkm3s}
We can write:
\begin{eqnarray}\label{eq:pr7}
sup(||S-\widetilde{S}_m||) \leq max (sup(||S-\widetilde{S'}_{m}||),
sup(||S-\widetilde{S''}_{m}||))\nonumber 
\end{eqnarray}
\end{cor} 

\begin{proof}
The proof is straightforward, if the Corollary \ref{con:msss} hold then
$sup(||S-\widetilde{S}_m||) = sup(||S-\widetilde{S''}_{m}||)$
else
$sup(||S-\widetilde{S}_m||) = sup(||S-\widetilde{S'}_{m}||)$, then in these both situations the $sup(||S-\widetilde{S}_m||)$ is at least smaller then the $max (sup(||S-\widetilde{S'}_{m}||), sup(||S-\widetilde{S''}_{m}||))$, and then proves the Corollary \ref{cor:boundkm3s}.
\end{proof}

\section{Evaluation}
We tested CMS3 and CMS3-tuned, and compared their performance to the results of four leading sampling methods described in Section \ref{section:RW}. These are:
\begin{itemize}
\item Random sampling (RS), 
\item K-means sampling (KS) \cite{zhang2008improved}, 
\item Minimum similarity sampling (SS) \cite{zeng2014minimum}, and 
\item Minimum sum of squared similarity sampling (MS3) \cite{bouneffouf2015sampling}. 
\end{itemize}
Notice that, we compare our algorithm to ensemble Nystr\"{o}m rather than the standard Nystr\"{o}m, since it was shown earlier both in \cite{kumar2009ensemble} and in  \cite{bouneffouf2016ensemble} that ensemble performs better than standard Nystr\"{o}m. We denote these algorithms as ensemble-RS, ensemble-KS, ensemble-SS, and ensemble-MS3, respectively.

We required each algorithm to sample 2\%, 4\%, 6\%, 8\% and 10\% of the data as landmark points, which are used by Nystr\"{o}m-based spectral clustering methods to cluster the datasets. We have also tested the ensemble Nystr\"{o}m methods with different values of $p$, going from $2$ to $10$.

Because sampling algorithms are sensitive to the datasets used, and clustering algorithms contain a degree of randomness, we used various benchmark datasets, and repeated our evaluations 10 times. We measured the clustering quality of each algorithm using their average accuracy across these tests, also recording their standard deviations.
\begin{table*}[t]
\centering
\caption{Accuracy on UCI Datasets}
\resizebox{1.64\columnwidth}{!}{
   \begin{tabular}{ l | l | l | l | l | l | l | l }
         &  Ensemble-SS     & 	Ensemble-KS  & 	       Ensemble-RS  & 	  	Ensemble-MS3      & \textbf{CMS3}           & \textbf{CMS3-tuned}\\ \hline
UCI Datasets \\ \hline
Abalone  & 	$84.82\pm 0.27$& 	$85.74\pm 0.31$& 	$85.69\pm 0.25$& 		$86.44\pm 0.40 $  & 	$88.21\pm 0.42$& 	$\textbf{89.19}\pm \textbf{0.21}$\\
Breast   & 	$67.85\pm 0.33$& 	$67.85\pm 0.32$& 	$67.83\pm 0.34$& 	$67.89 \pm 0.32$  & 	$68.94\pm 1.17$& 	$\textbf{70.55} \pm \textbf{0.34}$\\
Wine     & 	$53.16\pm 1.73$& 	$55.17\pm 3.80$& 	$54.78\pm 3.50$&  	$67.99\pm 3.67$   & 	$70.9\pm 1.87$& 	$\textbf{71.39} \pm \textbf{2.01}$\\
Wdbc     & 	$51.32 \pm 0.13$ & $51.31 \pm 0.13$ & $51.30 \pm 0.12$&     $51.32 \pm 0.14$ &   $\textbf{52.98} \pm \textbf{0.07}$ &  $ 52.31 \pm 0.13$\\
Yeast    & 	$67.58\pm 0.13$& 	$66.70\pm 0.13$& 	$66.92\pm 0.12$& 	$67.87\pm 0.12$   & 	$69.06\pm 0.07$& 	$\textbf{69.55} \pm \textbf{0.07}$\\
Shuttle    & 	$39.82\pm 2.51$& 	$37.90\pm 2.89$& 	$37.87\pm 2.23$& 	$41.45\pm 3.85 $   & 	$44.02\pm 1.84$& 	$\textbf{44.31} \pm \textbf{1.94}$\\
Letter    & 	$41.77\pm 1.83$& 	$40.34\pm 9.69$& 	$38.66\pm 9.77$& 	$ 53.32\pm 1.02$   & 	$56.44\pm 3.70$& 	$\textbf{57.64} \pm \textbf{3.81}$\\
PenDigits &$56.55\pm 0.16 $&     $56.46\pm 0.21$& 	$56.46\pm 0.22$ & 	$56.94\pm 0.19$   & 	$\textbf{58.08}\pm \textbf{0.18}$& 	$57.88\pm 0.39$\\
a7a     & 	$16.45\pm 1.17$&    $21.18\pm  4.92$&   $22.06\pm 4.08$&    $27.04\pm 1.45$  &   $25.89\pm 3.14$  &   $\textbf{26.18} \pm \textbf{3.10}$\\

   \end{tabular}
   }
     \label{table:AccuSyn} 
 \end{table*}
\subsection{Comparison on Accuracy}
We compared the performance of the seven sampling methods using data from University of California, Irvine (UCI) Machine Learning Repository\footnote{https://archive.ics.uci.edu/ml/datasets.html}. We chose nine datasets with different Instances, attributes and classes size: Abalone, Breast, Wine, Wdbc, Yeast, Shuttle, Letter, PenDigits and a7a. A brief summary of the datasets is listed in Table \ref{table:Synthetic}.

Table \ref{table:AccuSyn} reports the average accuracy of each algorithm, along with their standard deviations across 1000 tests on the UCI datasets. As expected, the accuracies depend on the dataset. For example, the accuracy of all algorithms in Haberman problem and Wdbc datasets stay in the range of $50\%$, while going as high as over $89\%$ for the Abalone dataset. From this observation, we can say that the Haberman problem and Wdbc datasets present difficulties to Nystr\"{o}m method-based spectral clustering.
 
We note that, on these datasets, all tested algorithms have better performance than the baseline of random sampling.
The results show that CMS3-tuned provided better clustering than the other algorithms on seven out of nine datasets, coming only narrowly second to CMS3 on the remaining two, though still within a standard deviation. Ranking the algorithms with respect to their mean accuracies, we note that the top two performing algorithms were CMS3-tuned and CMS3, in that order. The results on the UCI dataset confirm our heuristics that choosing between CMS3 and MS3 need to be done according to the spectrum shape of the dataset.
\begin{table}[H]
\centering
\caption{Datasets used for benchmarking}
\resizebox{0.78\columnwidth}{!}{
   \begin{tabular}{ l | c | r | r }
     UCI Datasets             & Instances & Attributes & Classes \\ \hline
          Abalone                  & 1484      & 7          &  3\\
          Breast                   & 699       & 9          &  2\\
          Wine                     & 178       & 13         &  3\\
          Wdbc                     & 569       & 32         &  2\\
          Yeast                    & 1484      & 6          &  8\\
          Shuttle                  & 14500     & 9          &  7\\
          Letter                   & 20000     & 16         &  26\\
          PenDigits                & 10992     & 16         &  10\\
          a7a                      & 16100     & 122        &  2\\
   \end{tabular}
   }
    \label{table:Synthetic} 
 \end{table}

The results of the Ensemble-SS algorithm show overall better performance compared to Ensemble-KS and Ensemble-RS sampling. We also notice that the ensemble-MS3 gave higher performance than the sampling algorithms that are not based on MS3.

\section{Conclusion}
In this paper, we introduced a new sampling algorithm for Nystr\"{o}m method-based spectral clustering, CMS3, and a heuristics on how it can be selected over the MS3 algorithm on which it is built. We call the latter CMS3-tuned. What sets CMS3 apart from other algorithms is that it uses the eigenspectrum of the input datasets to choose between sampling algorithms; CMS3 and MS3 in this case.
Further, through benchmarking experiments we have demonstrated the favourable performance of our algorithms.
\section*{Acknowledgment}
The authors thank Dr. Inanc Birol for his help and advice.

\ifCLASSOPTIONcaptionsoff
  \newpage
\fi

\bibliographystyle{ieeetr}
\bibliography{biblio}

\begin{thebibliography}{10}

\bibitem{fowlkes2004spectral}
C.~Fowlkes, S.~Belongie, F.~Chung, and J.~Malik, ``Spectral grouping using the
  nystrom method,'' {\em Pattern Analysis and Machine Intelligence, IEEE
  Transactions on}, vol.~26, no.~2, pp.~214--225, 2004.

\bibitem{BouneffoufBG13}
D.~Bouneffouf, A.~Bouzeghoub, and A.~L. Gan{\c{c}}arski, ``Risk-aware
  recommender systems,'' in {\em Neural Information Processing - 20th
  International Conference, {ICONIP} 2013, Daegu, Korea, November 3-7, 2013.
  Proceedings, Part {I}} (M.~Lee, A.~Hirose, Z.~Hou, and R.~M. Kil, eds.),
  vol.~8226 of {\em Lecture Notes in Computer Science}, pp.~57--65, Springer,
  2013.

\bibitem{ChoromanskaCKLR19}
A.~Choromanska, B.~Cowen, S.~Kumaravel, R.~Luss, M.~Rigotti, I.~Rish,
  P.~Diachille, V.~Gurev, B.~Kingsbury, R.~Tejwani, and D.~Bouneffouf, ``Beyond
  backprop: Online alternating minimization with auxiliary variables,'' in {\em
  Proceedings of the 36th International Conference on Machine Learning, {ICML}
  2019, 9-15 June 2019, Long Beach, California, {USA}} (K.~Chaudhuri and
  R.~Salakhutdinov, eds.), vol.~97 of {\em Proceedings of Machine Learning
  Research}, pp.~1193--1202, {PMLR}, 2019.

\bibitem{RiemerKBF19}
M.~Riemer, T.~Klinger, D.~Bouneffouf, and M.~Franceschini, ``Scalable
  recollections for continual lifelong learning,'' in {\em The Thirty-Third
  {AAAI} Conference on Artificial Intelligence, {AAAI} 2019, The Thirty-First
  Innovative Applications of Artificial Intelligence Conference, {IAAI} 2019,
  The Ninth {AAAI} Symposium on Educational Advances in Artificial
  Intelligence, {EAAI} 2019, Honolulu, Hawaii, USA, January 27 - February 1,
  2019}, pp.~1352--1359, {AAAI} Press, 2019.

\bibitem{LinC0RR20}
B.~Lin, G.~A. Cecchi, D.~Bouneffouf, J.~Reinen, and I.~Rish, ``A story of two
  streams: Reinforcement learning models from human behavior and
  neuropsychiatry,'' in {\em Proceedings of the 19th International Conference
  on Autonomous Agents and Multiagent Systems, {AAMAS} '20, Auckland, New
  Zealand, May 9-13, 2020} (A.~E.~F. Seghrouchni, G.~Sukthankar, B.~An, and
  N.~Yorke{-}Smith, eds.), pp.~744--752, International Foundation for
  Autonomous Agents and Multiagent Systems, 2020.

\bibitem{lin2020online}
B.~Lin, D.~Bouneffouf, and G.~Cecchi, ``Online learning in iterated prisoner's
  dilemma to mimic human behavior,'' {\em arXiv preprint arXiv:2006.06580},
  2020.

\bibitem{lin2020unified}
B.~Lin, G.~Cecchi, D.~Bouneffouf, J.~Reinen, and I.~Rish, ``Unified models of
  human behavioral agents in bandits, contextual bandits and rl,'' {\em arXiv
  preprint arXiv:2005.04544}, 2020.

\bibitem{NoothigattuBMCM19}
R.~Noothigattu, D.~Bouneffouf, N.~Mattei, R.~Chandra, P.~Madan, K.~R. Varshney,
  M.~Campbell, M.~Singh, and F.~Rossi, ``Teaching {AI} agents ethical values
  using reinforcement learning and policy orchestration,'' in {\em Proceedings
  of the Twenty-Eighth International Joint Conference on Artificial
  Intelligence, {IJCAI} 2019, Macao, China, August 10-16, 2019} (S.~Kraus,
  ed.), pp.~6377--6381, ijcai.org, 2019.

\bibitem{surveyDB}
D.~Bouneffouf and I.~Rish, ``A survey on practical applications of multi-armed
  and contextual bandits,'' {\em CoRR}, vol.~abs/1904.10040, 2019.

\bibitem{LR85}
T.~L. Lai and H.~Robbins, ``Asymptotically efficient adaptive allocation
  rules,'' {\em Advances in Applied Mathematics}, vol.~6, no.~1, pp.~4--22,
  1985.

\bibitem{Bouneffouf0SW19}
D.~Bouneffouf, S.~Parthasarathy, H.~Samulowitz, and M.~Wistuba, ``Optimal
  exploitation of clustering and history information in multi-armed bandit,''
  in {\em Proceedings of the Twenty-Eighth International Joint Conference on
  Artificial Intelligence, {IJCAI} 2019, Macao, China, August 10-16, 2019}
  (S.~Kraus, ed.), pp.~2016--2022, ijcai.org, 2019.

\bibitem{LinBCR18}
B.~Lin, D.~Bouneffouf, G.~A. Cecchi, and I.~Rish, ``Contextual bandit with
  adaptive feature extraction,'' in {\em 2018 {IEEE} International Conference
  on Data Mining Workshops, {ICDM} Workshops, Singapore, Singapore, November
  17-20, 2018} (H.~Tong, Z.~J. Li, F.~Zhu, and J.~Yu, eds.), pp.~937--944,
  {IEEE}, 2018.

\bibitem{DB2019}
A.~Balakrishnan, D.~Bouneffouf, N.~Mattei, and F.~Rossi, ``Incorporating
  behavioral constraints in online {AI} systems,'' {\em AAAI 2019}, 2019.

\bibitem{BalakrishnanBMR19ibm}
A.~Balakrishnan, D.~Bouneffouf, N.~Mattei, and F.~Rossi, ``Using multi-armed
  bandits to learn ethical priorities for online {AI} systems,'' {\em {IBM}
  Journal of Research and Development}, vol.~63, no.~4/5, pp.~1:1--1:13, 2019.

\bibitem{BouneffoufLUFA14}
D.~Bouneffouf, R.~Laroche, T.~Urvoy, R.~Feraud, and R.~Allesiardo, ``Contextual
  bandit for active learning: Active thompson sampling,'' in {\em Neural
  Information Processing - 21st International Conference, {ICONIP} 2014,
  Kuching, Malaysia, November 3-6, 2014. Proceedings, Part {I}}, pp.~405--412,
  2014.

\bibitem{RLbd2018}
R.~Noothigattu, D.~Bouneffouf, N.~Mattei, R.~Chandra, P.~Madan, K.~R. Varshney,
  M.~Campbell, M.~Singh, and F.~Rossi, ``Interpretable multi-objective
  reinforcement learning through policy orchestration,'' {\em CoRR},
  vol.~abs/1809.08343, 2018.

\bibitem{balakrishnan2020constrained}
A.~Balakrishnan, D.~Bouneffouf, N.~Mattei, and F.~Rossi, ``Constrained
  decision-making and explanation of a recommendation,'' Jan.~16 2020.
\newblock US Patent App. 16/050,176.

\bibitem{BouneffoufRCF17}
D.~Bouneffouf, I.~Rish, G.~A. Cecchi, and R.~F{\'{e}}raud, ``Context attentive
  bandits: Contextual bandit with restricted context,'' in {\em IJCAI 2017,
  Melbourne, Australia, August 19-25, 2017}, pp.~1468--1475, 2017.

\bibitem{BalakrishnanBMR18}
A.~Balakrishnan, D.~Bouneffouf, N.~Mattei, and F.~Rossi, ``Using contextual
  bandits with behavioral constraints for constrained online movie
  recommendation,'' in {\em Proceedings of the Twenty-Seventh International
  Joint Conference on Artificial Intelligence, {IJCAI} 2018, July 13-19, 2018,
  Stockholm, Sweden.}, pp.~5802--5804, 2018.

\bibitem{BouneffoufBG12}
D.~Bouneffouf, A.~Bouzeghoub, and A.~L. Gan{\c{c}}arski, ``A contextual-bandit
  algorithm for mobile context-aware recommender system,'' in {\em Neural
  Information Processing - 19th International Conference, {ICONIP} 2012, Doha,
  Qatar, November 12-15, 2012, Proceedings, Part {III}} (T.~Huang, Z.~Zeng,
  C.~Li, and C.~Leung, eds.), vol.~7665 of {\em Lecture Notes in Computer
  Science}, pp.~324--331, Springer, 2012.

\bibitem{Bouneffouf16}
D.~Bouneffouf, ``Exponentiated gradient exploration for active learning,'' {\em
  Computers}, vol.~5, no.~1, p.~1, 2016.

\bibitem{aaai0G20}
S.~Liu, P.~Ram, D.~Vijaykeerthy, D.~Bouneffouf, G.~Bramble, H.~Samulowitz,
  D.~Wang, A.~Conn, and A.~G. Gray, ``An {ADMM} based framework for automl
  pipeline configuration,'' in {\em The Thirty-Fourth {AAAI} Conference on
  Artificial Intelligence, {AAAI} 2020, The Thirty-Second Innovative
  Applications of Artificial Intelligence Conference, {IAAI} 2020, The Tenth
  {AAAI} Symposium on Educational Advances in Artificial Intelligence, {EAAI}
  2020, New York, NY, USA, February 7-12, 2020}, pp.~4892--4899, {AAAI} Press,
  2020.

\bibitem{AllesiardoFB14}
R.~Allesiardo, R.~F{\'{e}}raud, and D.~Bouneffouf, ``A neural networks
  committee for the contextual bandit problem,'' in {\em Neural Information
  Processing - 21st International Conference, {ICONIP} 2014, Kuching, Malaysia,
  November 3-6, 2014. Proceedings, Part {I}}, pp.~374--381, 2014.

\bibitem{dj2020}
D.~Bouneffouf and E.~Claeys, ``Hyper-parameter tuning for the contextual
  bandit,'' {\em CoRR}, vol.~abs/2005.02209, 2020.

\bibitem{Sohini2019}
S.~Upadhyay, M.~Agarwal, D.~Bouneffouf, and Y.~Khazaeni, ``A bandit approach to
  posterior dialog orchestration under a budget,'' {\em CoRR},
  vol.~abs/1906.09384, 2019.

\bibitem{bouneffouf2020online}
D.~Bouneffouf, ``Online learning with corrupted context: Corrupted contextual
  bandits,'' {\em arXiv preprint arXiv:2006.15194}, 2020.

\bibitem{bouneffouf2020contextual}
D.~Bouneffouf, S.~Upadhyay, and Y.~Khazaeni, ``Contextual bandit with missing
  rewards,'' {\em arXiv e-prints}, pp.~arXiv--2007, 2020.

\bibitem{kong2011fast}
T.~Kong, Y.~Tian, and H.~Shen, ``A fast incremental spectral clustering for
  large data sets,'' in {\em Parallel and Distributed Computing, Applications
  and Technologies (PDCAT), 2011 12th International Conference on}, pp.~1--5,
  IEEE, 2011.

\bibitem{ChenC11}
X.~Chen and D.~Cai, ``Large scale spectral clustering with landmark-based
  representation,'' in {\em Proceedings of the Twenty-Fifth {AAAI} Conference
  on Artificial Intelligence, {AAAI} 2011, San Francisco, California, USA,
  August 7-11, 2011}, 2011.

\bibitem{zhang2011clusterability}
X.~Zhang and Q.~You, ``Clusterability analysis and incremental sampling for
  nystr{\"o}m extension based spectral clustering,'' in {\em Data Mining
  (ICDM), 2011 IEEE 11th International Conference on}, pp.~942--951, IEEE,
  2011.

\bibitem{williams2001using}
C.~Williams and M.~Seeger, ``Using the nystr{\"o}m method to speed up kernel
  machines,'' in {\em Proceedings of the 14th Annual Conference on Neural
  Information Processing Systems}, no.~EPFL-CONF-161322, pp.~682--688, 2001.

\bibitem{cohen2014uniform}
M.~B. Cohen, Y.~T. Lee, C.~Musco, C.~Musco, R.~Peng, and A.~Sidford, ``Uniform
  sampling for matrix approximation,'' {\em arXiv preprint arXiv:1408.5099},
  2014.

\bibitem{bouneffouf2016ensemble}
D.~Bouneffouf and I.~Birol, ``Ensemble minimum sum of squared similarities
  sampling for nystr{\"o}m-based spectral clustering,'' in {\em Neural Networks
  (IJCNN), 2016 International Joint Conference on}, pp.~3851--3855, IEEE, 2016.

\bibitem{bouneffouf2015sampling}
D.~Bouneffouf and I.~Birol, ``Sampling with minimum sum of squared similarities
  for nystrom-based large scale spectral clustering,'' in {\em Proceedings of
  the 23rd International Joint Conference on Artificial Intelligence}, 2015.

\bibitem{kumar2009ensemble}
S.~Kumar, M.~Mohri, and A.~Talwalkar, ``Ensemble nystrom method,'' in {\em
  Advances in Neural Information Processing Systems}, pp.~1060--1068, 2009.

\bibitem{zeng2014minimum}
Z.~Zeng, M.~Zhu, H.~Yu, and H.~Ma, ``Minimum similarity sampling scheme for
  nystr{\"o}m based spectral clustering on large scale high-dimensional data,''
  in {\em Modern Advances in Applied Intelligence}, pp.~260--269, Springer,
  2014.

\bibitem{sloan1981quadrature}
I.~H. Sloan, ``Quadrature methods for integral equations of the second kind
  over infinite intervals,'' {\em Mathematics of computation}, vol.~36,
  no.~154, pp.~511--523, 1981.

\bibitem{Loo2014}
Z.~Fu, ``Optimal landmark selection for nyström approximation,'' in {\em
  Neural Information Processing} (C.~Loo, K.~Yap, K.~Wong, A.~Teoh, and
  K.~Huang, eds.), vol.~8835 of {\em Lecture Notes in Computer Science},
  pp.~311--318, Springer International Publishing, 2014.

\bibitem{kumar2009sampling}
S.~Kumar, M.~Mohri, and A.~Talwalkar, ``On sampling-based approximate spectral
  decomposition,'' in {\em Proceedings of the 26th Annual International
  Conference on Machine Learning}, pp.~553--560, ACM, 2009.

\bibitem{drineas2005Nystrom}
P.~Drineas and M.~W. Mahoney, ``On the nystr{\"o}m method for approximating a
  gram matrix for improved kernel-based learning,'' {\em The Journal of Machine
  Learning Research}, vol.~6, pp.~2153--2175, 2005.

\bibitem{drineas2006fast}
P.~Drineas, R.~Kannan, and M.~W. Mahoney, ``Fast monte carlo algorithms for
  matrices ii: Computing a low-rank approximation to a matrix,'' {\em SIAM
  Journal on Computing}, vol.~36, no.~1, pp.~158--183, 2006.

\bibitem{belabbas2009spectral}
M.-A. Belabbas and P.~J. Wolfe, ``Spectral methods in machine learning and new
  strategies for very large datasets,'' {\em Proceedings of the National
  Academy of Sciences}, vol.~106, no.~2, pp.~369--374, 2009.

\bibitem{gowda2010schur}
M.~S. Gowda and R.~Sznajder, ``Schur complements, schur determinantal and
  haynsworth inertia formulas in euclidean jordan algebras,'' {\em Linear
  Algebra and Its Applications}, vol.~432, no.~6, pp.~1553--1559, 2010.

\bibitem{zhang2008improved}
K.~Zhang, I.~W. Tsang, and J.~T. Kwok, ``Improved nystr{\"o}m low-rank
  approximation and error analysis,'' in {\em Proceedings of the 25th
  international conference on Machine learning}, pp.~1232--1239, ACM, 2008.

\bibitem{shinnou2008spectral}
H.~Shinnou and M.~Sasaki, ``Spectral clustering for a large data set by
  reducing the similarity matrix size.,'' in {\em LREC}, 2008.

\bibitem{von2007tutorial}
U.~Von~Luxburg, ``A tutorial on spectral clustering,'' {\em Statistics and
  computing}, vol.~17, no.~4, pp.~395--416, 2007.

\bibitem{macqueen1967some}
J.~B. MacQueen, ``Some methods for classification and analysis of multivariate
  observations,'' in {\em Proceedings of the fifth Berkeley symposium on
  Mathematical Statistics and Probability}, vol.~1, pp.~281--297, 1967.

\end{thebibliography}
\end{document}